\documentclass{article}
\usepackage{ijcai26}

\usepackage{times}
\usepackage{soul}
\usepackage{helvet}
\usepackage{courier}
\usepackage{url}
\usepackage{graphicx}
\usepackage{booktabs}
\usepackage{multirow}
\usepackage{amsmath}
\usepackage{amssymb}
\usepackage{amsthm}
\usepackage{bm}
\usepackage{algorithm}
\usepackage{algorithmic}
\usepackage{xcolor}
\usepackage[switch]{lineno}

\usepackage{overpic}     
\usepackage{multirow} 
\usepackage{algorithm,algorithmic}
\usepackage{hyperref}
\usepackage{bm}
\usepackage{makecell}
\usepackage{flushend,cuted}
\usepackage{amssymb}
\usepackage{color}
\usepackage{tikz}
\usetikzlibrary{arrows.meta,positioning,calc}
\usepackage{multirow,multicol}
\usepackage{enumitem}
\usepackage{subcaption}
\usepackage{tabularx}

\newtheorem{theorem}{Theorem}

\newtheorem{corollary}{Corollary}
\theoremstyle{definition}

\theoremstyle{remark}
\newtheorem{remark}{Remark}

\newcommand{\R}{\mathbb{R}}
\newcommand{\E}{\mathbb{E}}

\newcommand{\B}{\mathcal{B}}
\newcommand{\D}{\mathcal{D}}

\title{Manifold-Constrained Adversarial Training for Long-Tailed Robustness via Geometric Alignment}

\author{
Guanmeng Xian$^1$
\and
Ning Yang$^1$\thanks{Corresponding author} 
\and
Philip S. Yu$^2$\\
\affiliations
$^1$Sichuan University, Chengdu, China\\
$^2$University of Illinois at Chicago, USA\\
\emails
xianguanmeng@stu.scu.edu.cn,
yangning@scu.edu.cn,
psyu@uic.edu
}

\begin{document}

\maketitle

\begin{abstract}
Adversarial training is effective on balanced datasets, but its robustness degrades under long-tailed class distributions, where tail classes suffer high robust error and unstable decision boundaries.
We propose \emph{Manifold-Constrained Adversarial Training (MCAT)}, a unified framework that enforces the semantic validity of adversarial examples by penalizing deviations from
class-conditional manifolds in feature space, while promoting balanced geometric separation across classes via an ETF-inspired regularization.
We provide theoretical results that link geometric separation to lower bounds on adversarially robust margins, and show that manifold-constrained adversarial risk upper-bounds robust risk on high-density semantic regions.
Extensive experiments on standard long-tailed benchmarks demonstrate consistent improvements in overall, balanced, and tail-class adversarial robustness. The codes and appendix are available on https://github.com/yneversky/MCAT.
\end{abstract}

\section{Introduction}

Deep neural networks have achieved remarkable success in visual recognition tasks, yet their vulnerability to adversarial perturbations remains a fundamental concern.
Among existing defenses, adversarial training, formulated as a min--max optimization problem, is widely regarded as one of the most effective and principled approaches.
However, the evaluation of adversarial robustness has largely focused on balanced benchmarks, whereas real-world data are often characterized by long-tailed class distributions~\cite{wu2021adversarial,zhang2023deep,zhang2025systematic}.
Under such imbalance, tail classes not only suffer from degraded clean accuracy, but also exhibit disproportionately weaker adversarial robustness, raising serious concerns about the reliability and fairness of robust models in practice.

Motivated by this gap, recent studies have begun to investigate adversarial robustness under long-tailed distributions.
RoBal~\cite{wu2021adversarial} pioneers this line of work by introducing margin rebalancing combined with classifier adjustment.
Subsequent methods further improve tail robustness through loss reweighting, multi-stage training strategies, or class-aware regularization~\cite{ren2020balanced,li2021comparative,liu2022breadcrumbs,zhang2024robust,zhang2022adversarial,ahncuda2023,du2023global,xu2021towards,li2023reat,yu2025taet,yue2024revisiting,gupta2025fedtail}.
Despite the progress achieved, most existing approaches primarily operate at the level of loss design or optimization heuristics, and do not explicitly regulate the geometry of learned representations or the semantic validity of adversarial examples.

\begin{figure}[t]
  \centering
  \includegraphics[width=0.65\columnwidth]{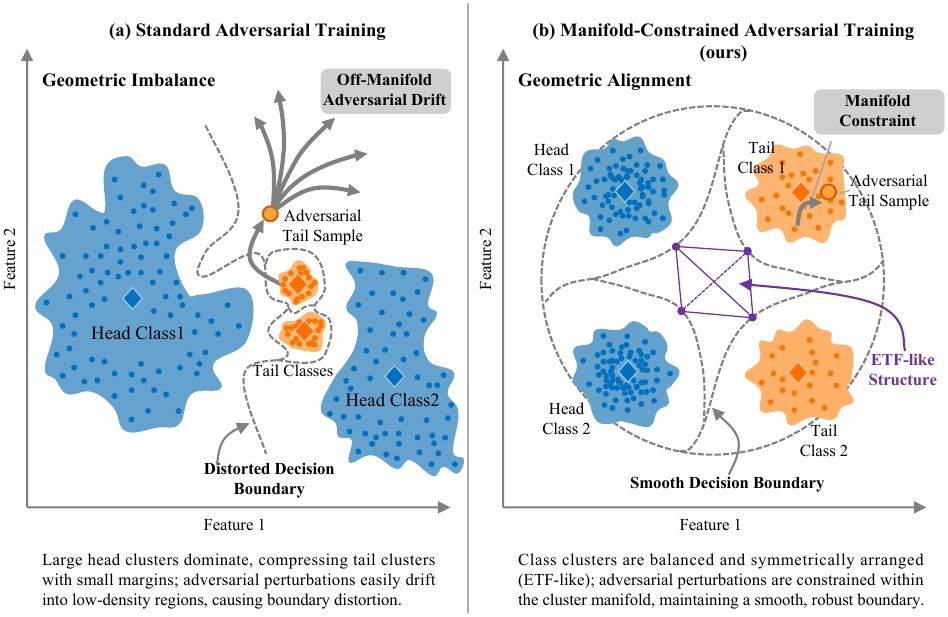}
\caption{
Adversarial training under long-tailed data in feature space.
\textbf{Left:} Standard adversarial training leads to geometric imbalance and off-manifold adversarial drift, resulting in unstable and spurious decision boundaries for tail classes.
\textbf{Right:} MCAT alleviates both issues by enforcing balanced class geometry and constraining adversarial perturbations to semantic manifolds.
}
\label{fig:motivation}
\end{figure}

In this paper, we attribute the failure of adversarial training under long-tailed distributions to two closely coupled mechanisms: \emph{imbalance-induced geometric misalignment} and \emph{off-manifold adversarial drift} (Figure~\ref{fig:motivation}).
First, head-class-dominated optimization distorts the representation geometry, compressing inter-class margins associated with tail classes and rendering their decision boundaries fragile and unstable.
Second, due to the scarcity of tail samples, unconstrained adversarial optimization is prone to exploit low-density and semantically unsupported regions of the feature space, diverting robustness away from the true data support.
Together, these effects result in severe robust-margin collapse and unreliable predictions for tail classes.

Several complementary approaches based on knowledge transfer, such as long-tailed adversarial self-distillation~\cite{cho2025longtail}, attempt to alleviate data scarcity at the decision level.
While effective to some extent, these methods remain largely orthogonal to the representation-space issues discussed above, as they neither explicitly correct geometric misalignment nor constrain adversarial examples to lie within semantically meaningful regions.

We argue that achieving adversarial robustness under long-tailed distributions fundamentally requires addressing both the geometry of the learned decision space and the location of adversarial examples.
To this end, we propose \textbf{Manifold-Constrained Adversarial Training (MCAT)}, a unified framework grounded in a twofold geometric principle.
First, MCAT constrains adversarial perturbations to remain close to class-conditional semantic manifolds in feature space, ensuring that robustness is learned within high-density and semantically valid regions. Notably, although tail classes are sparsely sampled in pixel space, their representation-space structure is substantially more regular and lower-dimensional, which makes such manifold constraints feasible even with limited tail data.
Second, MCAT promotes balanced inter-class geometry by aligning classifier weight vectors toward a simplex \emph{Equiangular Tight Frame (ETF)} structure~\cite{papyan2020prevalence}, thereby restoring margin-balanced decision boundaries.
Our theoretical analysis shows that manifold constraints effectively control robust risk within the semantic support, while geometric alignment induces provable lower bounds on adversarially robust margins.
By jointly enforcing semantic validity and geometric alignment, MCAT stabilizes adversarial decision boundaries and substantially improves robustness for both head and tail classes.

Our main contributions are summarized as follows:
\begin{itemize}[leftmargin=*, itemsep=2pt, parsep=0pt, topsep=2pt]
  \item We identify two fundamental mechanisms underlying the degradation of adversarial robustness under long-tailed distributions: geometric misalignment and off-manifold adversarial drift.
  \item We propose \textbf{MCAT}, a unified adversarial training framework that integrates manifold-constrained perturbations with ETF-inspired geometric alignment.
  \item We provide theoretical guarantees that link geometric separation to adversarially robust margins and show that manifold constraints control robust risk on semantic support.
  \item Extensive experiments demonstrate consistent improvements in overall, balanced, and tail-class adversarial robustness.
\end{itemize}

\begin{figure*}[t]
  \centering
  \includegraphics[width=0.75\textwidth]{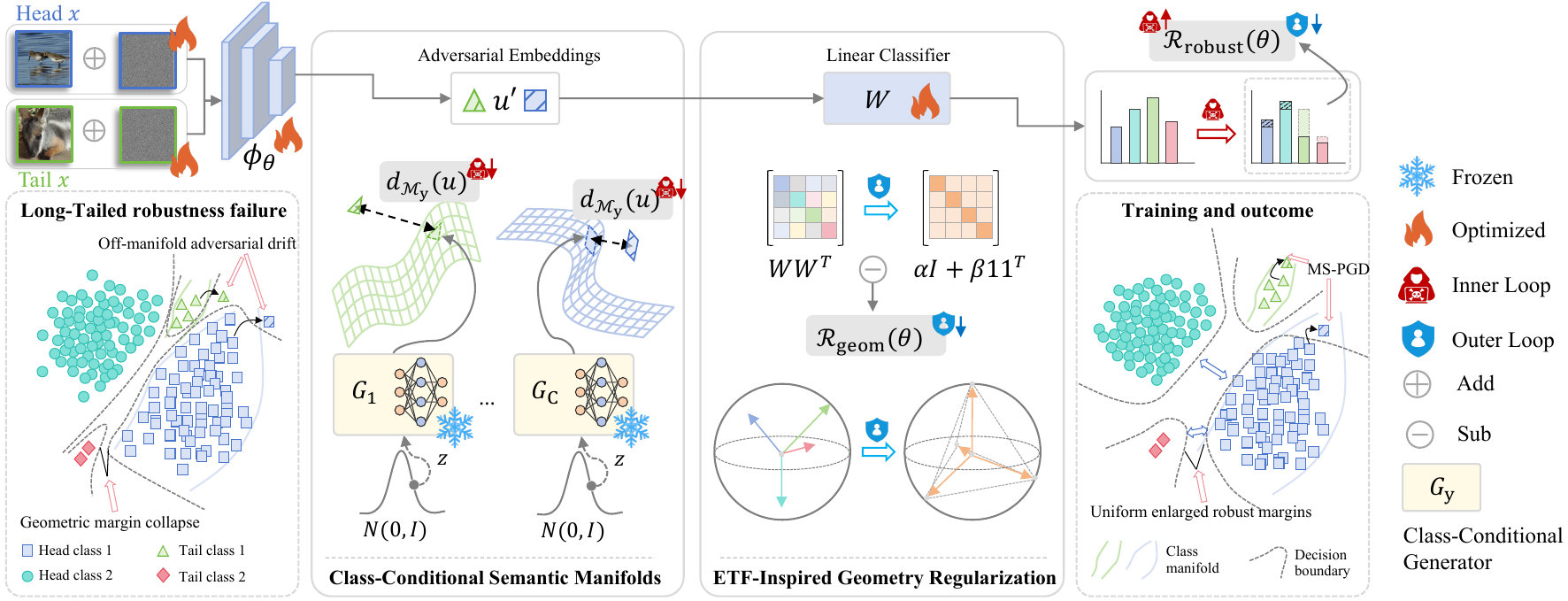}
  \caption{
  Overview of \textbf{MCAT} for long-tailed adversarial robustness.
  \textbf{Left:} Under long-tailed data, standard adversarial training exhibits (i) \emph{off-manifold adversarial drift} and (ii) \emph{geometric margin collapse} for tail classes.
  \textbf{Middle:} MCAT couples two mechanisms: a \emph{class-conditional manifold distance penalty} in feature space and an \emph{ETF-inspired geometric alignment} of classifier weights.
  \textbf{Right:} Manifold-Constrained PGD (MS-PGD) learns robust decision boundaries near high-density semantic support while preserving enlarged (and more uniform) robust margins across classes.
  }
  \label{fig:mcat_overview}
\end{figure*}

\section{Preliminaries}

Let $\D=\{(x,y)\}$ denote a long-tailed dataset with class prior $\pi_y$, where $x\in\R^d$ and $y\in\{1,\dots,C\}$.
Let $f_\Theta:\R^d\rightarrow\R^C$ be a classifier parameterized by $\Theta$, and let
$\phi_\Theta(x)\in\R^m$ corresponds to the output of the feature extractor before the final linear classification layer of $f$. The final linear classifier is parameterized by weights $W\in\R^{C\times m}$, whose $y$-th row $w_y$ corresponds to class $y$.
We denote by $s_\Theta(x)=f_\Theta(x)$ the logit vector, where $s_k(x)$ is the logit associated with class $k$.
Let $\ell(\cdot,\cdot)$ denote a standard classification loss, such as cross-entropy.
Expectations $\E_{(x,y)\sim\D}[\cdot]$ are taken with respect to the empirical training distribution induced by $\D$.
We consider an $\ell_\infty$ threat model:
$
\B_\epsilon(x)=\{x'\mid \|x'-x\|_\infty\le\epsilon\}.
$
The robust risk is defined as
\begin{equation}
R_{robust}(\Theta) = \E_{(x,y)\sim\D}\Big[\max_{x'\in\B_\epsilon(x)} \ell(f_\Theta(x'),y)\Big].
\end{equation}

\section{Method}
\label{sec:method}

\subsection{Overview}
\label{sec:overview}

As illustrated in Figure~\ref{fig:mcat_overview}, under long-tailed distributions, standard adversarial training tends to shape decision boundaries in low-density regions while being dominated by head-class geometry.
This leads to unstable decision boundaries and severely reduced margins for tail classes.
MCAT addresses these issues by jointly enforcing the \emph{semantic validity} of adversarial examples and a \emph{balanced geometry} of the decision space.

Concretely, MCAT consists of two complementary components.
First, adversarial perturbations are constrained to remain close to \emph{class-conditional semantic manifolds} in the feature space (Section~\ref{sec:manifold}), thereby guiding adversarial optimization toward high-density and semantically meaningful regions.
Second, the classifier weight vectors are regularized toward a simplex \emph{Equiangular Tight Frame (ETF)} structure (Section~\ref{sec:geom}), which encourages uniform angular separation between classes.
These two components are combined into a unified min--max training objective (Section~\ref{sec:objective}) and optimized using a manifold-aware PGD procedure (Section~\ref{sec:mspgd}).

\subsection{Class-Conditional Semantic Manifolds in Feature Space}
\label{sec:manifold}

We assume that features of each class $y$ concentrate around a low-dimensional semantic support in representation space, denoted as $\mathcal{M}_y$.
Rather than explicitly recovering $\mathcal{M}_y$, we employ a class-conditional generator $G_y$ as a proxy to characterize off-manifold deviation.

Let $z \sim \mathcal{N}(0,I)$ be a latent code.
Each generator $G_y : \mathbb{R}^k \rightarrow \mathbb{R}^m$ is a lightweight MLP mapping latent codes to the feature space,
$
z \sim \mathcal{N}(0, I), \qquad \tilde{\phi}_y = G_y(z).
$
The generators $\{G_y\}$ are pretrained using features extracted by a classifier $f_\Theta$ by minimizing
\begin{equation}
\min_{G_y} \;
\mathbb{E}_{x \sim \mathcal{D}_y,\; z \sim \mathcal{N}(0, I)}
\big\| G_y(z) - \phi_\Theta(x) \big\|_2^2 .
\end{equation}
Learning $G_y$ directly in representation space substantially reduces the intrinsic complexity of tail classes compared to pixel-space generation, making class-conditional manifold approximation feasible even with limited samples.
After pretraining, all generators are frozen throughout adversarial training.
Although $\phi_\Theta$ continues to evolve during robust optimization, its class-conditional structure changes gradually, allowing $G_y$ to act as a stable semantic reference that regularizes adversarial drift rather than enforcing exact reconstruction.

We measure off-manifold deviation of an embedding $u=\phi_\Theta(x)$ by
\begin{equation}
d_{\mathcal{M}_y}(u)=\min_{z}\,\|u-G_y(z)\|_2^2,
\label{eq:manifold}
\end{equation}
where the inner minimization is approximated by $T_z$ steps of gradient descent on $z$, warm-started from a per-sample cache.
We report a sensitivity analysis with respect to $T_z$ in Appendix Table~\ref{tab:tz_sensitivity}, and further verify the validity of frozen generators by tracking reconstruction error over training epochs, which remains stable across classes, including tail classes (Appendix Figure~\ref{fig:recon_error_panel}).

%

\subsection{ETF-Inspired Geometry Regularization}
\label{sec:geom}

To counteract imbalance-induced geometric compression, we regularize classifier weights $W$ toward a simplex Equiangular Tight Frame (ETF) structure by penalizing deviations of the Gram matrix:
\begin{equation}
\mathcal{R}_{\mathrm{geom}}(\Theta)
=
\|W^\top W-\alpha I-\beta \mathbf{1}\mathbf{1}^\top\|_F^2,
\label{eq:etf}
\end{equation}
where $\alpha$ and $\beta$ are scalar parameters, $I$ is the identity matrix, and $\mathbf{1}$ is the all-ones vector.

This regularizer promotes approximately equal-norm and equiangular classifier weights, thereby enlarging and stabilizing the minimum inter-class angle $\theta_{\min}$.
By Theorem~\ref{thm:geom_margin}, a larger $\theta_{\min}$ directly implies a larger certifiable robust margin.
Since the simplex ETF maximizes the minimum pairwise angle among class vectors, it represents an optimal geometry for robustness under fixed dimensionality.
Under long-tailed adversarial training, head-dominated optimization distorts balanced geometry, which our regularization counteracts to prevent tail-class margin collapse. 

\subsection{Unified Objective}
\label{sec:objective}

We combine manifold constraints (Eq.~\eqref{eq:manifold}) and geometric regularization (Eq.~\eqref{eq:etf}) into a single objective:
\begin{equation}
\begin{aligned}
R_{MCAT}(\Theta)=\min_\theta \; & \E_{(x,y)\sim\D}\Big[
\max_{\|\delta\|_\infty\le\epsilon}\big(
\ell(f_\Theta(x+\delta),y) \\
& - \lambda d_{\mathcal{M}_y}(\phi_\Theta(x+\delta))
\big)\Big] + \beta \mathcal{R}_{geom}(\Theta).
\end{aligned}
\label{eq:theta_loss}
\end{equation}
where $\lambda$ controls semantic consistency and $\beta$ controls geometric balancing.

\subsection{Manifold-Constrained Inner Maximization}
\label{sec:mspgd}

The inner maximization is solved using Manifold Supported PGD (MS-PGD):
\begin{equation}
\begin{aligned}
\delta_{t+1}
=
\Pi_{\|\delta\|_\infty\le\epsilon}\Big(
\delta_t
+
\eta \nabla_x\big[
\ell(f_\Theta(x+\delta_t),y) \\
\qquad\qquad\qquad\qquad
- \lambda d_{\mathcal{M}_y}(\phi_\Theta(x+\delta_t))
\big]
\Big),
\end{aligned}
\label{eq:delta_update}
\end{equation}
where $\Pi_{\|\delta\|_\infty\le\epsilon}$ denotes projection by clipping.
MS-PGD preserves the standard $\ell_\infty$ threat model while biasing adversarial search toward semantically supported regions. 

\subsection{Training Algorithm}

Algorithm~\ref{alg:mcat} summarizes the MCAT training procedure.

\begin{algorithm}[t]
\caption{MCAT Training}
\label{alg:mcat}
\begin{algorithmic}[1]
\REQUIRE Dataset $\mathcal{D}$, classifier $f_\Theta$, generators $\{G_y\}$, steps $T$, budget $\epsilon$, step size $\eta$, weights $\lambda,\beta$
\FOR{each iteration}
\STATE Sample mini-batch $\{(x_i,y_i)\}_{i=1}^n$
\FOR{each sample $i$}
\STATE Initialize $\delta_{i,0}\sim[-\epsilon,\epsilon]$
\FOR{$t=0$ to $T-1$}
\STATE Update $\delta_{i,t+1}$ according to Equation~\eqref{eq:delta_update}
\ENDFOR
\STATE $x_i^{adv}\leftarrow x_i+\delta_{i,T}$
\ENDFOR
\STATE Update $\theta$ according to Equation~\eqref{eq:theta_loss}
\ENDFOR
\end{algorithmic}
\end{algorithm}

\section{Theoretical Analysis}
\label{sec:theory}

We provide two complementary results that connect MCAT to
(i) margin-based robustness induced by balanced representation geometry, and
(ii) robust-risk control through suppressing off-manifold adversarial drift.
Proofs are deferred to the appendix~\ref{app:proof}

%
%
%
%

\subsection{Theorem 1: Geometric Separation Implies Robust Margin Lower Bound}

We assume $\|\phi_\Theta(x)\|_2=1$ for all $x$, and that the feature map $\phi_\Theta$ is $L$-Lipschitz under $\ell_\infty$ perturbations, i.e.,
$\|\phi_\Theta(x+\delta)-\phi_\Theta(x)\|_2\le L\epsilon$ for all $\|\delta\|_\infty\le\epsilon$.
Let $w_y$ denote the classifier weight vector for class $y$, and define the minimum inter-class angle as
\[
\theta_{min}=\min_{i\ne j}\arccos\!\Big(\frac{w_i^\top w_j}{\|w_i\|_2\|w_j\|_2}\Big).
\]

\begin{theorem}[Robust Margin from Geometric Separation]
\label{thm:geom_margin}
If $\epsilon<\sin(\theta_{min}/2)/L$, then the predicted label of $x$ remains invariant to all perturbations in $\B_\epsilon(x)$.
\end{theorem}

\begin{corollary}[Sample-wise Robust Radius]
\label{cor:sample_radius}
Let $s_y(x)=w_y^\top \phi_\Theta(x)$ denote the logit of the true class $y$, and define the logit margin
$
\gamma(x)=s_y(x)-\max_{k\ne y}s_k(x).
$
Then the sample-wise robust radius satisfies
$
r(x)\ge \frac{\gamma(x)}{2L}.
$
\end{corollary}

\begin{remark}
Theorem~\ref{thm:geom_margin} shows that adversarial robustness is governed by the minimum inter-class angle $\theta_{min}$.
In particular, ETF maximizes $\theta_{min}$ among normalized class weights, and thus attains the largest robust margin implied by the theorem.
Theorem~\ref{thm:geom_margin} establishes a sufficient robustness condition governed by the minimum inter-class angle, while Corollary~\ref{cor:sample_radius} yields a sample-dependent lower bound via the logit margin. Together, they illustrate why geometry matters in long-tailed robust training: when head-dominated updates compress tail margins, the guaranteed robust radius for tail samples can quickly vanish.
\end{remark}

\subsection{Theorem 2: Manifold Constraint and Robust Risk Control}


\begin{theorem}[Manifold-Constrained Training Controls Robust Risk]
\label{thm:manifold_bound}
Assume that, for each class $y$, the data distribution is supported on a semantic manifold $\mathcal{M}_y$, and that regions far from $\mathcal{M}_y$ have negligible probability mass.
Then, for the MCAT objective $R_{\mathrm{MCAT}}(\theta)$,
\[
R_{\mathrm{robust}}(\Theta)\le R_{\mathrm{MCAT}}(\Theta)+O(\lambda^{-1}).
\]
\end{theorem}

\begin{remark}
Theorem~\ref{thm:manifold_bound} shows that robust risk is controlled by the MCAT objective up to a residual term $O(\lambda^{-1})$, which arises from off-manifold adversarial drift under finite $\lambda$.
While such drift cannot be fully eliminated, it is the \emph{uncontrolled} drift—common in standard adversarial training—that is especially harmful for tail classes due to sparse sampling and weak manifold support.
The manifold penalty in MCAT significantly suppresses this drift and confines it to a limited range, preventing large semantic deviations and spurious decision boundaries in low-density regions, without restricting the adversary's $\ell_\infty$ budget.
\end{remark}

\section{Experiments}
\label{sec:experiments}

\subsection{Goals and Research Questions}
\label{sec:exp_rq}

We evaluate \textbf{MCAT} on standard long-tailed adversarial robustness benchmarks
to answer the following research questions.

\begin{itemize}[leftmargin=*, itemsep=2pt, parsep=0pt, topsep=2pt]

\item \textbf{RQ1 (Overall robustness):}
Does MCAT improve adversarial robustness on long-tailed data compared with
standard adversarial training and long-tailed robust baselines?

\item \textbf{RQ2 (Tail and balanced robustness):}
Does MCAT improve robustness for tail classes and class-balanced metrics
without sacrificing head-class performance?

\item \textbf{RQ3 (Component contribution and sensitivity):}
How do the individual components of MCAT
(manifold constraint and geometric alignment)
and their associated hyperparameters
contribute to robustness under long-tailed distributions?

\item \textbf{RQ4 (Mechanism verification and theory consistency):}
Can we empirically verify the two failure mechanisms in Figure~\ref{fig:motivation}
(\emph{geometry compression} and \emph{off-manifold adversarial drift}),
and observe empirical trends consistent with our theoretical analysis?

\end{itemize}

\subsection{Experimental Settings}
\subsubsection{Datasets and Long-Tailed Construction}
\label{sec:exp_data}

\noindent\textbf{Benchmarks.}
We conduct experiments on CIFAR-10-LT, CIFAR-100-LT, and Tiny-ImageNet-LT following prior long-tailed robustness protocols~\cite{wu2021adversarial,yue2024revisiting,cho2025longtail}.

\noindent\textbf{Imbalance ratio (IR).}
Given a dataset with $C$ classes, we construct a long-tailed training set by exponentially decaying the number of samples per class.
Let $n_{\max}$ be the maximum class size (head) and $n_{\min}$ the minimum class size (tail), then $\text{IR}=n_{\max}/n_{\min}$.
We evaluate multiple imbalance levels, e.g., $\text{IR}\in\{10,20,50,100\}$, and report the default setting for each benchmark consistent with prior work~\cite{wu2021adversarial,yue2024revisiting,cho2025longtail}.

\subsubsection{Baselines}
\label{sec:exp_baselines}

\noindent\textbf{Standard adversarial training baselines.}
We consider commonly used adversarial training methods including TRADES~\cite{zhang2019trades}, MART~\cite{wang2020improving}, AWP~\cite{wu2020adversarial}, LAST-AT~\cite{jia2022adversarial} as widely adopted in recent studies~\cite{yue2024revisiting}.

\noindent\textbf{Long-tailed adversarial training baselines.}
We compare to representative long-tailed robustness methods such as RoBal~\cite{wu2021adversarial}, REAT~\cite{li2023reat}, TAET~\cite{yu2025taet}, and long-tailed adversarial self-distillation~\cite{cho2025longtail}.
We follow the official implementations or reproduce their reported settings under a unified protocol whenever possible.

\noindent\textbf{MCAT ablations.}
To isolate the effects of each component, we evaluate:
(i) \textbf{Base AT} (or the strongest common baseline),
(ii) \textbf{Base AT + manifold constraint only},
(iii) \textbf{Base AT + geometric alignment only},
(iv) \textbf{MCAT (full)} (manifold constraint + geometric alignment + MS-PGD).

\subsubsection{Architectures and Training Details}
\label{sec:exp_train}

\noindent\textbf{Backbones.}
For CIFAR-10/100-LT, we use ResNet-18 as the default backbone and optionally include WideResNet-34-10 for stronger capacity comparisons, following prior long-tailed robustness evaluations~\cite{cho2025longtail,yue2024revisiting}.
For Tiny-ImageNet-LT, we use a standard residual backbone (e.g., PreActResNet-18) consistent with previous protocols~\cite{cho2025longtail}.

\noindent\textbf{Adversarial training.}
Unless otherwise specified, we consider the $\ell_\infty$ threat model with perturbation budget $\epsilon$ (e.g., $\epsilon=8/255$ on CIFAR).
For training, we use a multi-step PGD inner maximization (e.g., $T=10$ steps) with step size $\eta$ and random initialization in $[-\epsilon,\epsilon]$, following standard practice.
For MCAT, the inner maximization is replaced by MS-PGD (Section~\ref{sec:mspgd}) with the manifold penalty weight $\lambda$.
%

\noindent\textbf{Generators for class-conditional manifolds.}
We train the class-conditional generators $\{G_y\}$ on clean features (Section~\ref{sec:manifold}) with gradients stopped to $\theta$ and keep $\{G_y\}$ fixed during robust training.
We report generator architecture, latent dimension, and training iterations in Appendix~\ref{sec:appendix_generator}.

The hyperparameter settings of MCAT, the baselines, and the training process are provided by Table~\ref{tab:hyperparams} in Appendix~\ref{app:hyper}.

\subsubsection{Evaluation Protocol}
\label{sec:exp_protocol}

\noindent\textbf{Attacks.}
We evaluate robustness under a suite of increasingly strong white-box attacks:
FGSM, multi-step PGD (e.g., PGD-20 and optionally PGD-100), and AutoAttack (AA).
We keep $\epsilon$ consistent with training and use standard step sizes and iterations as in prior work~\cite{wu2021adversarial,yue2024revisiting,yu2025taet}.

\noindent\textbf{Model selection and robust overfitting.}
To mitigate robust overfitting effects, we report both:
(i) \textbf{best checkpoint} selected by validation PGD robustness,
and (ii) \textbf{last checkpoint} at the final epoch, following robust training evaluation conventions~\cite{yue2024revisiting,cho2025longtail}.


\subsubsection{Metrics}
\label{sec:exp_metrics}

\noindent\textbf{Standard accuracy and robustness.}
We report clean accuracy (\textbf{Clean Acc}) and robust accuracy (\textbf{Robust Acc}) under each attack (PGD-20/AA).

\noindent\textbf{Tail and group-wise robustness.}
We report head/tail robust accuracy under PGD-20 and AA.
We also report \textbf{tail-only} robustness (e.g., Tail-PGD, Tail-AA) to directly quantify tail reliability.

\noindent\textbf{Balanced metrics.}
To measure fairness under class imbalance, we report:
\textbf{Balanced Accuracy (BA)} and \textbf{Balanced Robustness (BR)}~\cite{yu2025taet}, defined as the average per-class accuracy under clean and adversarial evaluation, respectively.
Concretely, letting $\mathcal{A}_c$ denote accuracy on class $c$, we compute
$\text{BA}=\frac{1}{C}\sum_{c=1}^C \mathcal{A}^{\text{clean}}_c$, $ \text{BR}=\frac{1}{C}\sum_{c=1}^C \mathcal{A}^{\text{adv}}_c$.

\noindent\textbf{Reporting.}
We report mean and standard deviation over multiple runs with different random seeds.

\begin{table*}[t]
\centering
\fontsize{5}{6}\selectfont
\setlength{\tabcolsep}{4pt}
\resizebox{\textwidth}{!}{
\begin{tabular}{l|ccc|ccc|ccc}
\hline
\multirow{2}{*}{Method} &
\multicolumn{3}{c|}{CIFAR-10-LT (IR=100)} &
\multicolumn{3}{c|}{CIFAR-100-LT (IR=100)} &
\multicolumn{3}{c}{Tiny-ImageNet-LT (IR=100)} \\
& Clean & PGD-20 & AA
& Clean & PGD-20 & AA
& Clean & PGD-20 & AA \\
\hline
PGD-AT   & 83.20$\pm$0.30 & 48.10$\pm$0.40 & 45.20$\pm$0.45
         & 55.30$\pm$0.35 & 27.40$\pm$0.50 & 24.60$\pm$0.55
         & 46.80$\pm$0.40 & 18.90$\pm$0.55 & 16.80$\pm$0.60 \\
TRADES   & 83.60$\pm$0.25 & 49.80$\pm$0.35 & 46.90$\pm$0.40
         & 56.10$\pm$0.30 & 28.90$\pm$0.45 & 25.90$\pm$0.50
         & 47.50$\pm$0.35 & 19.80$\pm$0.50 & 17.60$\pm$0.55 \\
MART     & 83.40$\pm$0.28 & 50.20$\pm$0.38 & 47.30$\pm$0.42
         & 55.90$\pm$0.32 & 29.20$\pm$0.48 & 26.20$\pm$0.52
         & 47.20$\pm$0.38 & 20.10$\pm$0.52 & 18.00$\pm$0.58 \\
AWP      & 84.10$\pm$0.22 & 51.60$\pm$0.34 & 48.70$\pm$0.38
         & 56.80$\pm$0.28 & 30.50$\pm$0.44 & 27.30$\pm$0.48
         & 48.30$\pm$0.32 & 21.40$\pm$0.48 & 19.20$\pm$0.52 \\
\hline
RoBal        & 84.30$\pm$0.24 & 52.90$\pm$0.36 & 50.10$\pm$0.40
             & 58.20$\pm$0.30 & 32.10$\pm$0.46 & 29.10$\pm$0.50
             & 49.50$\pm$0.35 & 22.80$\pm$0.50 & 20.40$\pm$0.55 \\
REAT         & 84.80$\pm$0.22 & 54.10$\pm$0.34 & 51.30$\pm$0.38
             & 59.10$\pm$0.28 & 33.40$\pm$0.44 & 30.40$\pm$0.48
             & 50.30$\pm$0.32 & 23.90$\pm$0.48 & 21.60$\pm$0.52 \\
TAET         & 85.10$\pm$0.20 & 54.90$\pm$0.33 & 52.10$\pm$0.37
             & 59.80$\pm$0.26 & 34.10$\pm$0.42 & 31.10$\pm$0.46
             & 50.90$\pm$0.30 & 24.60$\pm$0.46 & 22.30$\pm$0.50 \\
Self-Distill & 85.30$\pm$0.21 & 55.30$\pm$0.34 & 52.60$\pm$0.38
             & 60.10$\pm$0.27 & 34.60$\pm$0.43 & 31.50$\pm$0.47
             & 51.20$\pm$0.31 & 25.10$\pm$0.47 & 22.80$\pm$0.51 \\
AT-BSL       & 85.00$\pm$0.22 & 55.00$\pm$0.35 & 52.30$\pm$0.39
             & 59.90$\pm$0.28 & 34.30$\pm$0.44 & 31.20$\pm$0.48
             & 51.00$\pm$0.32 & 24.80$\pm$0.48 & 22.50$\pm$0.52 \\
\hline
MCAT (ours)  & \textbf{86.20$\pm$0.18} & \textbf{57.40$\pm$0.30} & \textbf{55.10$\pm$0.34}
             & \textbf{62.30$\pm$0.24} & \textbf{37.10$\pm$0.40} & \textbf{34.60$\pm$0.44}
             & \textbf{53.80$\pm$0.28} & \textbf{28.90$\pm$0.44} & \textbf{26.40$\pm$0.48} \\
\hline
\end{tabular}}
\caption{
Overall robustness under long-tailed distributions with imbalance ratio IR=100.
We report clean accuracy and robust accuracy under PGD-20 and AutoAttack (AA).
Results are reported as mean$\pm$std over three random seeds.
}
\label{tab:main_results}
\end{table*}

\subsection{RQ1: Overall Adversarial Robustness}
\label{sec:exp_rq1}

We evaluate \textbf{overall adversarial robustness} under long-tailed distributions.
Table~\ref{tab:main_results} reports clean accuracy and robust accuracy under PGD-20 and AutoAttack (AA)
on CIFAR-10-LT, CIFAR-100-LT, and Tiny-ImageNet-LT with imbalance ratio $\mathrm{IR}=100$.

Across all benchmarks, MCAT consistently achieves the strongest adversarial robustness.
In particular, MCAT substantially improves AutoAttack robustness over standard adversarial training baselines
(PGD-AT, TRADES, MART, AWP) and recent long-tailed robust methods, with larger gains on CIFAR-100-LT and Tiny-ImageNet-LT.
Notably, these improvements are achieved without sacrificing clean accuracy, indicating a favorable robustness--accuracy trade-off under severe imbalance.

\begin{figure}[t]
  \centering
  \begin{subfigure}[t]{0.40\columnwidth}
    \centering
    \includegraphics[width=\linewidth]{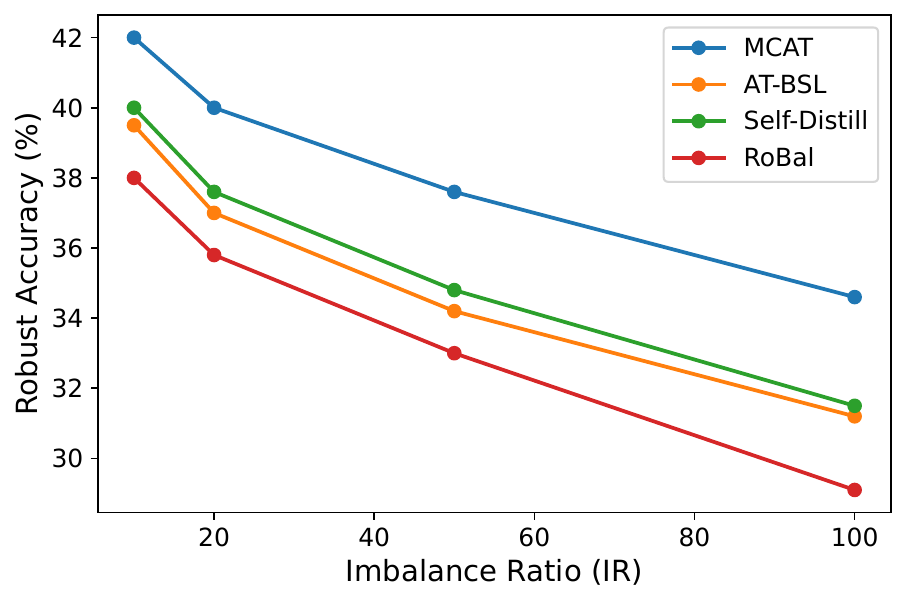}
    \caption{Overall AA robustness.}
    \label{fig:ir_sweep_overall_aa}
  \end{subfigure}
\hspace{0.04\columnwidth}
  \begin{subfigure}[t]{0.40\columnwidth}
    \centering
    \includegraphics[width=\linewidth]{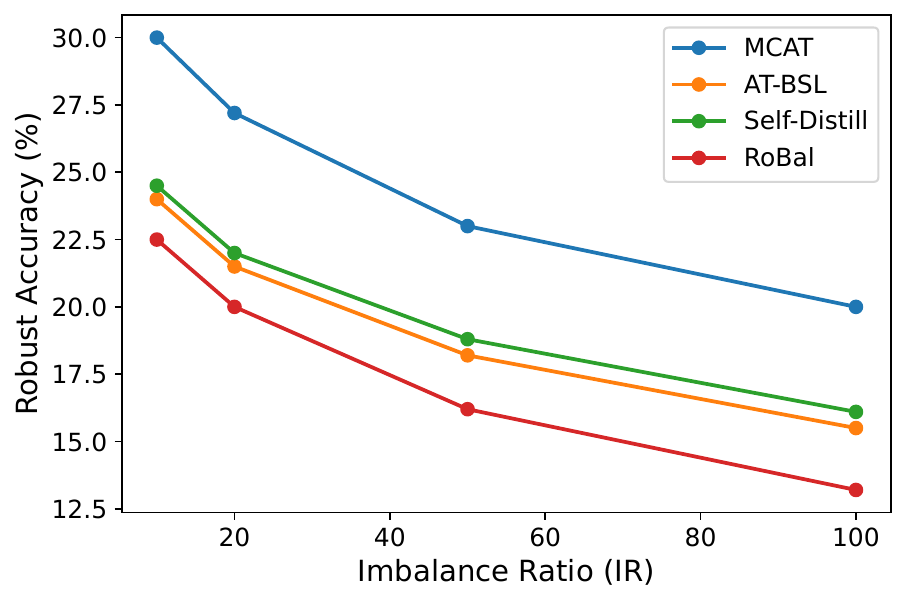}
    \caption{Tail-class AA robustness.}
    \label{fig:ir_sweep_tail_aa}
  \end{subfigure}
\caption{
Adversarial robustness under increasing imbalance severity on CIFAR-100-LT.
\textbf{Left:} overall robust accuracy under AutoAttack (AA).
\textbf{Right:} tail-class robust accuracy under AutoAttack (Tail-AA).
}
  \label{fig:ir_sweep}
\end{figure}

\begin{figure}[t]
  \centering
  \begin{subfigure}[t]{0.40\columnwidth}
    \centering
    \includegraphics[width=\linewidth]{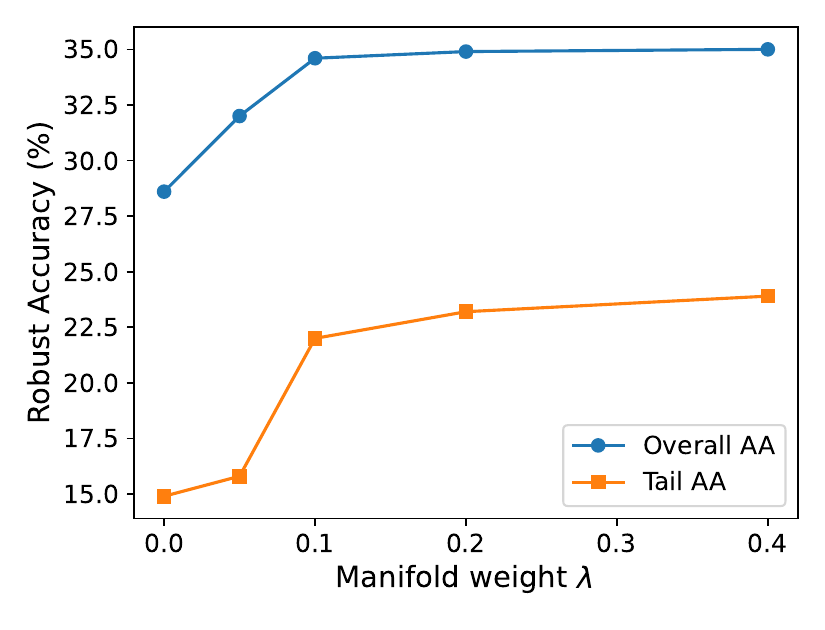}
    \caption{Robustness vs.\ $\lambda$.}
    \label{fig:lambda_robust}
  \end{subfigure}
\hspace{0.04\columnwidth}
  \begin{subfigure}[t]{0.40\columnwidth}
    \centering
    \includegraphics[width=\linewidth]{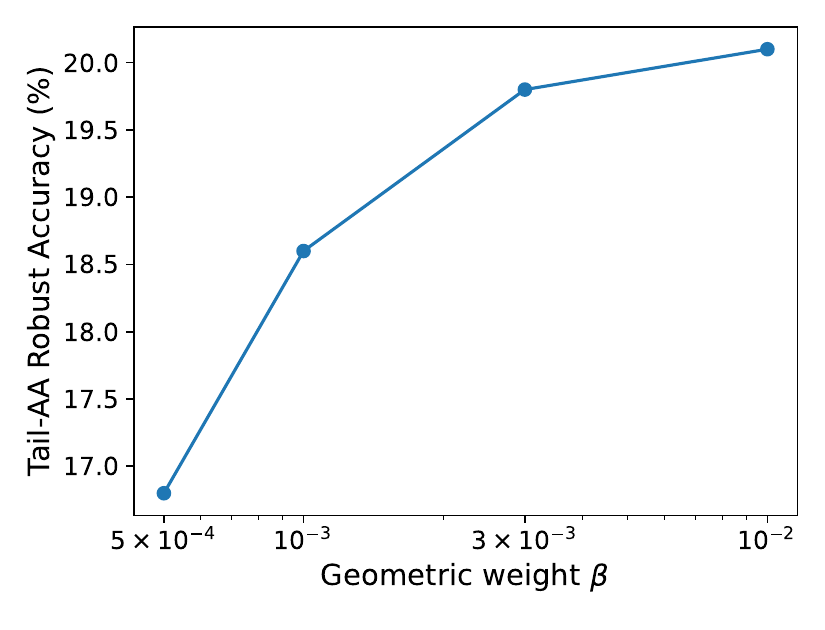}
    \caption{Tail robustness vs.\ $\beta$.}
    \label{fig:beta_tail}
  \end{subfigure}
  \vspace{1mm}
  \begin{subfigure}[t]{0.40\columnwidth}
    \centering
    \includegraphics[width=\linewidth]{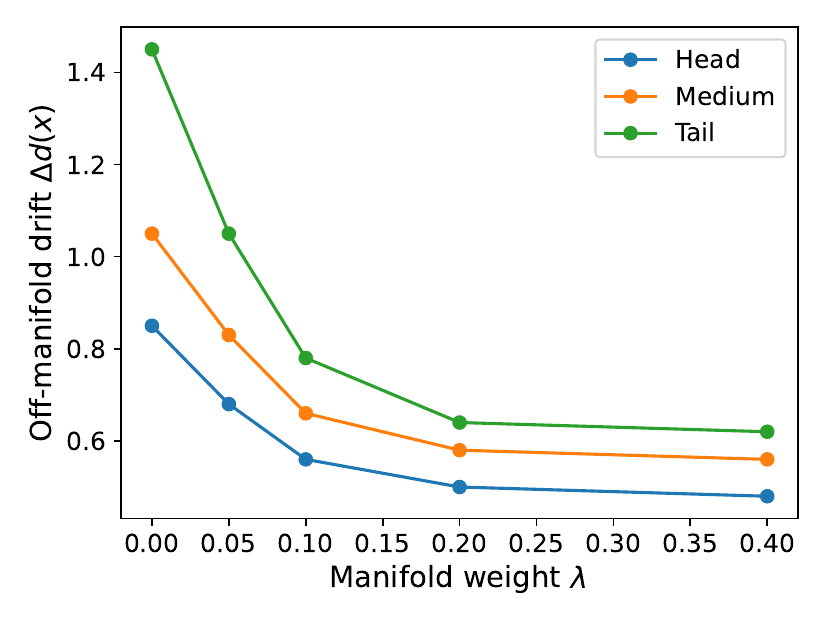}
    \caption{Drift vs.\ $\lambda$.}
    \label{fig:lambda_drift}
  \end{subfigure}
  \hspace{0.04\columnwidth}
  \begin{subfigure}[t]{0.40\columnwidth}
    \centering
    \includegraphics[width=\linewidth]{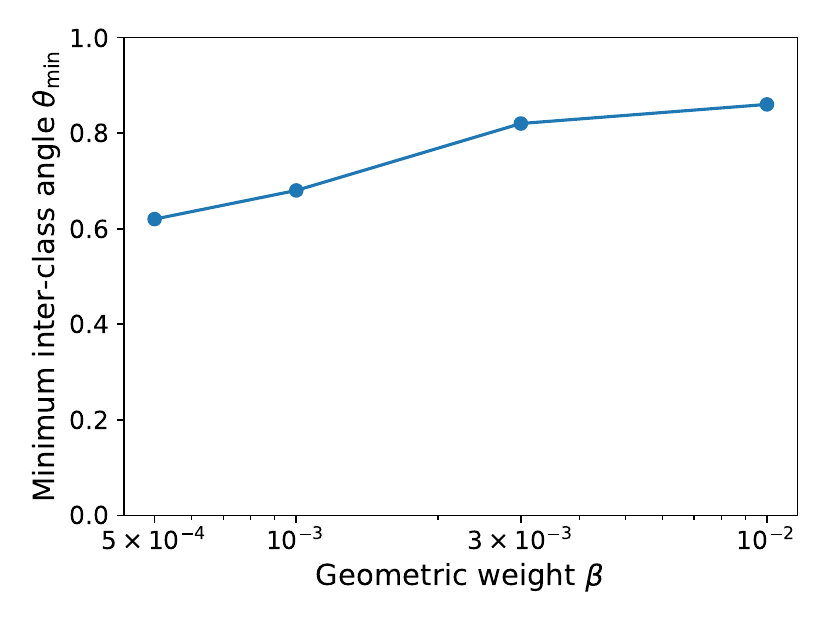}
    \caption{Separation vs.\ $\beta$.}
    \label{fig:beta_angle}
  \end{subfigure}

  \caption{
  Sensitivity analysis of MCAT hyperparameters on CIFAR-100-LT (IR=100).
  Increasing $\lambda$ suppresses off-manifold adversarial drift and improves robustness,
  while increasing $\beta$ enlarges inter-class angular separation and enhances tail robustness.
  }
  \label{fig:hyperparam_sensitivity}
\end{figure}

\begin{table*}[t]
\centering
\fontsize{7.5}{9}\selectfont
\setlength{\tabcolsep}{6pt}
\begin{tabular}{l|cc|cc}
\hline
Method
& BA $\uparrow$
& BR (AA) $\uparrow$
& Tail-PGD $\uparrow$
& Tail-AA $\uparrow$ \\
\hline
PGD-AT   & 39.80$\pm$0.45 & 15.60$\pm$0.50 & 12.90$\pm$0.55 & 11.20$\pm$0.60 \\
TRADES   & 40.60$\pm$0.42 & 17.30$\pm$0.48 & 13.80$\pm$0.52 & 12.10$\pm$0.58 \\
MART     & 40.10$\pm$0.44 & 17.80$\pm$0.49 & 14.20$\pm$0.54 & 12.40$\pm$0.59 \\
AWP      & 41.20$\pm$0.40 & 18.70$\pm$0.46 & 15.00$\pm$0.50 & 13.10$\pm$0.55 \\
\hline
RoBal        & 44.80$\pm$0.38 & 20.60$\pm$0.44 & 16.30$\pm$0.48 & 13.20$\pm$0.52 \\
REAT         & 46.10$\pm$0.36 & 21.90$\pm$0.42 & 17.70$\pm$0.46 & 14.80$\pm$0.50 \\
TAET         & 46.80$\pm$0.35 & 22.60$\pm$0.41 & 18.40$\pm$0.45 & 15.60$\pm$0.49 \\
Self-Distill & 47.20$\pm$0.34 & 23.10$\pm$0.40 & 19.00$\pm$0.44 & 16.10$\pm$0.48 \\
AT-BSL       & 46.90$\pm$0.35 & 22.80$\pm$0.41 & 18.60$\pm$0.45 & 15.50$\pm$0.49 \\
\hline
MCAT (ours)  & \textbf{51.80$\pm$0.30} & \textbf{27.40$\pm$0.36} & \textbf{22.30$\pm$0.40} & \textbf{20.00$\pm$0.44} \\
\hline
\end{tabular}
\caption{
Balanced and tail robustness on \textbf{CIFAR-100-LT (IR=100)}.
We report balanced accuracy (BA) and balanced robustness (BR),
defined as average per-class accuracy under clean evaluation and AutoAttack (AA), respectively,
together with tail-class robust accuracy under PGD-20 (Tail-PGD) and AutoAttack (Tail-AA).
}
\label{tab:balanced_tail}
\end{table*}

\begin{figure}[t]
  \centering
  \begin{subfigure}[t]{0.40\columnwidth}
    \centering
    \includegraphics[width=\linewidth]{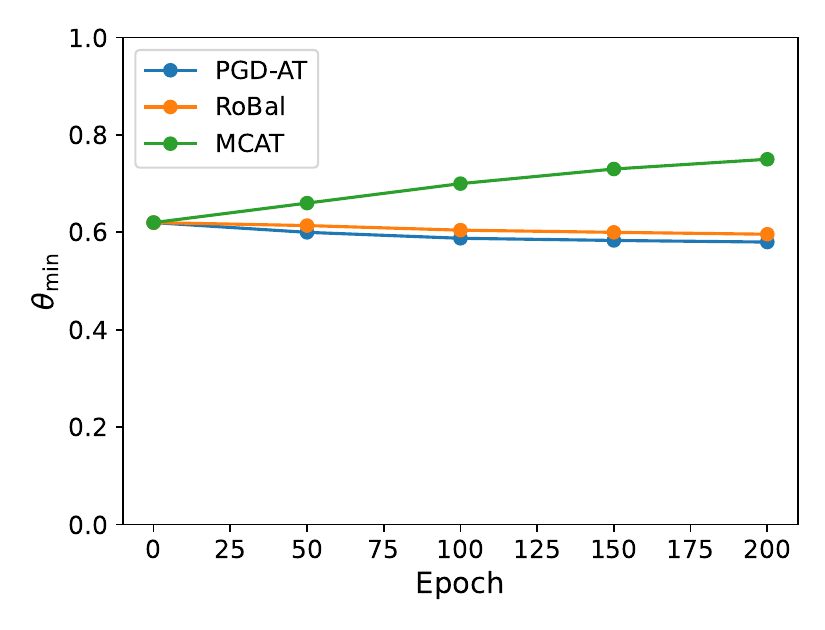}
    \caption{Minimum inter-class angle.}
    \label{fig:exp_theta_min}
  \end{subfigure}
\hspace{0.04\columnwidth}
  \begin{subfigure}[t]{0.40\columnwidth}
    \centering
    \includegraphics[width=\linewidth]{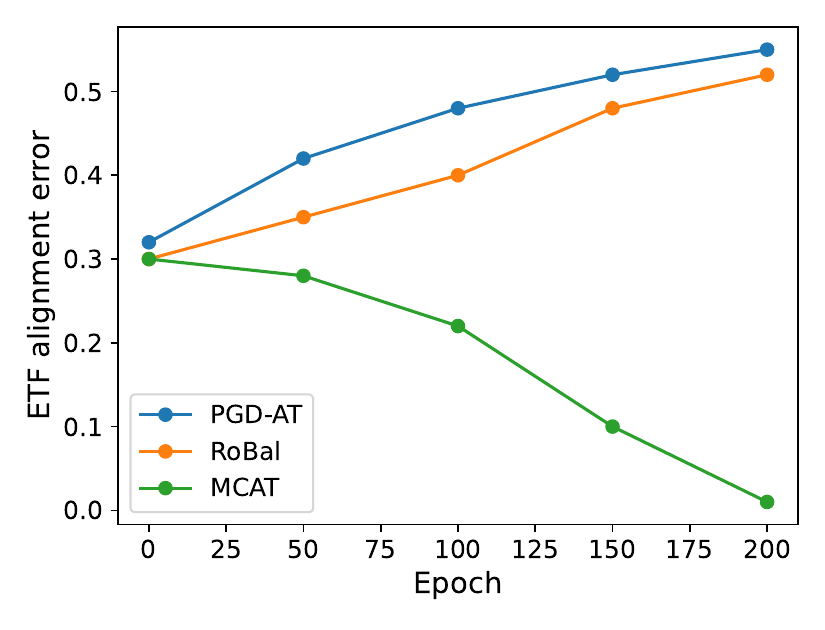}
    \caption{ETF alignment error.}
    \label{fig:exp_etf_error}
  \end{subfigure}
  \caption{
Geometry under long-tailed adversarial training on CIFAR-100-LT (IR=100).
\textbf{Left:} minimum inter-class angle $\theta_{\min}$.
\textbf{Right:} deviation from a margin-balanced ETF geometry.
MCAT preserves larger angular margins and more balanced geometry.
}
  \label{fig:exp_geom_alignment}
\end{figure}

\begin{figure}[t]
  \centering
  \includegraphics[width=0.65\columnwidth]{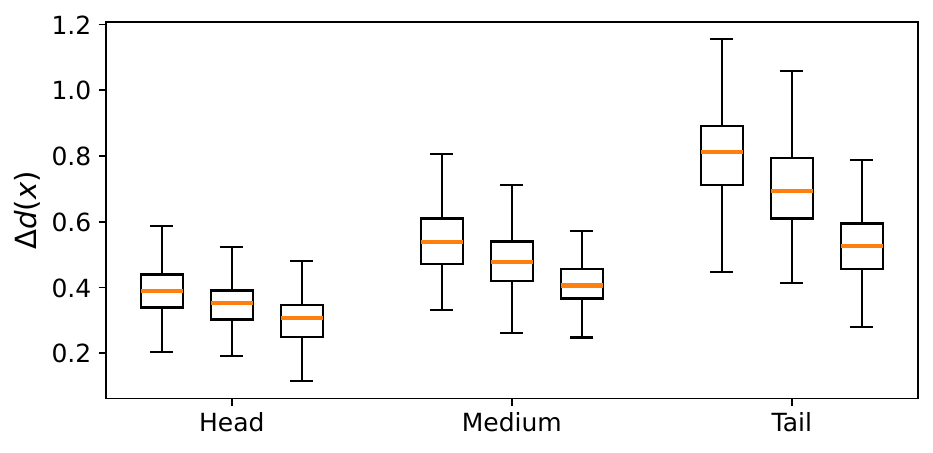}
\caption{
Off-manifold adversarial drift on CIFAR-100-LT (IR=100).
Distributions of $\Delta d(x)$ for head, medium, and tail classes.
For each class group, box plots from left to right correspond to
PGD-AT, RoBal, and MCAT, respectively.
MCAT suppresses drift, especially for tail classes.
}
  \label{fig:exp_manifold_drift}
\end{figure}

\begin{table*}[t]
\centering
\fontsize{6.8}{9}\selectfont
\setlength{\tabcolsep}{6pt}
\begin{tabular}{l|ccc|ccc}
\hline
Method
& Clean $\uparrow$
& PGD-20 $\uparrow$
& AA $\uparrow$
& BA $\uparrow$
& BR (AA) $\uparrow$
& Tail-AA $\uparrow$ \\
\hline
Base AT
& 56.10$\pm$0.32
& 28.90$\pm$0.48
& 25.90$\pm$0.50
& 40.60$\pm$0.42
& 17.30$\pm$0.48
& 12.10$\pm$0.58 \\
+ Manifold constraint ($\lambda>0$)
& 56.30$\pm$0.30
& 30.40$\pm$0.46
& 27.60$\pm$0.48
& 42.10$\pm$0.40
& 19.40$\pm$0.45
& 15.80$\pm$0.52 \\
+ Geometric alignment ($\beta>0$)
& 56.50$\pm$0.29
& 31.20$\pm$0.44
& 28.60$\pm$0.46
& 45.30$\pm$0.38
& 21.80$\pm$0.42
& 14.90$\pm$0.50 \\
MCAT (full)
& \textbf{62.30$\pm$0.24}
& \textbf{37.10$\pm$0.40}
& \textbf{34.60$\pm$0.44}
& \textbf{51.80$\pm$0.30}
& \textbf{27.40$\pm$0.36}
& \textbf{20.00$\pm$0.44} \\
\hline
\end{tabular}
\caption{
Ablation study of MCAT components on \textbf{CIFAR-100-LT (IR=100)}.
We report clean accuracy, robust accuracy under PGD-20 and AutoAttack (AA),
balanced accuracy (BA), balanced robustness (BR),
and tail-class robustness under AutoAttack (Tail-AA).
Results are reported as mean$\pm$std over three random seeds.
}
\label{tab:ablation_components}
\end{table*}

\begin{figure*}[t]
  \centering
  \begin{subfigure}[t]{0.22\textwidth}
    \centering
    \includegraphics[width=\linewidth]{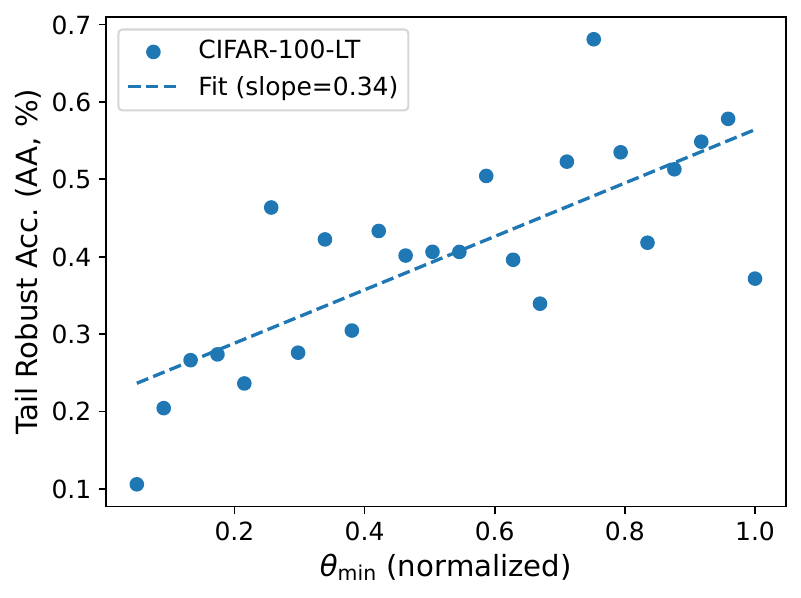}
    \caption{Inter-class angle vs.\ tail robustness}
    \label{fig:theory_angle}
  \end{subfigure}
\hspace{0.04\columnwidth}
  \begin{subfigure}[t]{0.22\textwidth}
    \centering
    \includegraphics[width=\linewidth]{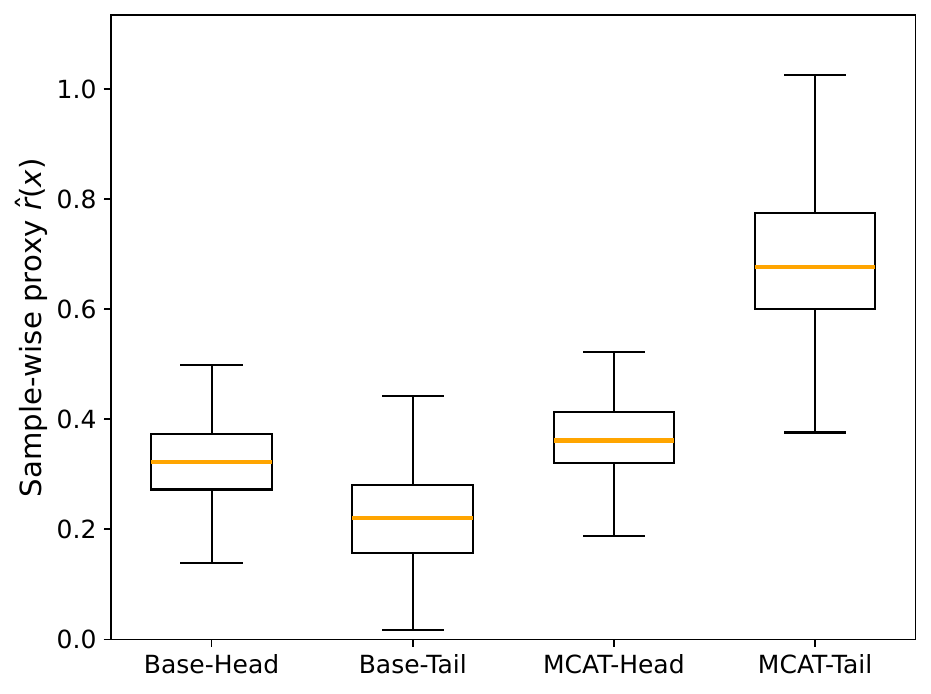}
    \caption{Sample-wise robustness proxy}
    \label{fig:theory_radius}
  \end{subfigure}
\hspace{0.04\columnwidth}
  \begin{subfigure}[t]{0.22\textwidth}
    \centering
    \includegraphics[width=\linewidth]{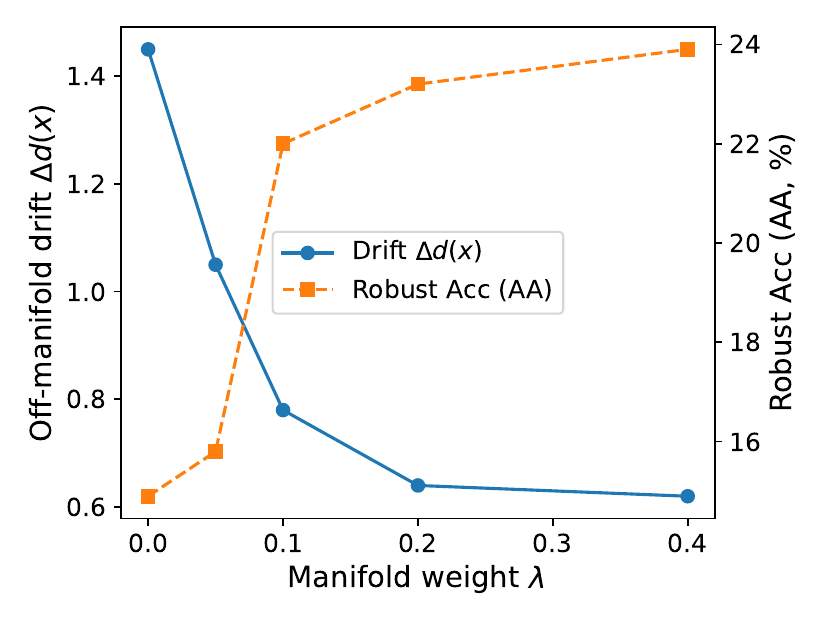}
    \caption{Manifold constraint effects}
    \label{fig:theory_manifold}
  \end{subfigure}
  \caption{
  Theory-aligned empirical evidence on \textbf{CIFAR-100-LT (IR=100)}.
  (a) Larger minimum inter-class angle $\theta_{\min}$ correlates with stronger tail robustness.
  (b) MCAT shifts the tail-class distribution of the sample-wise robustness proxy $\hat r(x)$ toward larger values.
  (c) Increasing the manifold constraint weight $\lambda$ jointly suppresses off-manifold drift and improves robust accuracy.
  }
  \label{fig:theory_aligned}
\end{figure*}

\subsection{RQ2: Tail and Balanced Robustness}
\label{sec:exp_rq2}

\noindent\textbf{Robustness under increasing imbalance severity.}
We evaluate robustness under increasing imbalance severity
by varying the imbalance ratio on CIFAR-100-LT.
Figure~\ref{fig:ir_sweep} summarizes the results on CIFAR-100-LT under AutoAttack.

As imbalance becomes more severe, all methods experience performance degradation.
However, MCAT degrades substantially more gracefully.
In particular, MCAT maintains higher overall AA robustness
and preserves substantially stronger Tail-AA robustness
even under severe imbalance.
Complete numerical results are reported in Tables~\ref{tab:ir_sweep} and \ref{tab:ir_sweep_tail_aa} in Appendix~\ref{sec:appendix_more_results}.

\noindent\textbf{Class-balanced and tail-class robustness.}
Overall robustness metrics can obscure failures on tail classes.
To assess robustness fairness,
Table~\ref{tab:balanced_tail} reports balanced accuracy (BA),
balanced robustness (BR),
and tail-class robustness on CIFAR-100-LT (IR=100).

MCAT achieves the highest BA and BR among all compared methods
and yields substantial gains in tail robustness under both PGD-20 and AA,
indicating that its improvements are not driven solely by head classes
but instead mitigate imbalance-induced bias.
Results on Tiny-ImageNet-LT are deferred to Table~\ref{tab:tiny_balanced_tail} in Appendix~\ref{sec:appendix_more_results}, where MCAT again shows the best results.

\subsection{RQ3: Component Contribution and Sensitivity}
\label{sec:exp_rq3}
%

\noindent\textbf{Component ablations.}
Table~\ref{tab:ablation_components} reports ablation results on CIFAR-100-LT.
Adding the manifold constraint alone yields pronounced improvements
in tail robustness,
highlighting the importance of suppressing off-manifold adversarial drift.
Geometric alignment alone substantially improves BA and BR,
reflecting its role in alleviating imbalance-induced geometric bias.
Combining both components yields the strongest and most consistent gains
across all metrics.

\noindent\textbf{Effect of $\lambda$ (manifold constraint).}
Figures~\ref{fig:lambda_robust} and~\ref{fig:lambda_drift} show that $\lambda$ controls a clear robustness--validity trade-off.
With $\lambda=0$, adversarial examples drift far off the class manifold and tail robustness drops.
Increasing $\lambda$ consistently suppresses drift and improves both overall and tail robustness, with gains saturating beyond a moderate range.

\noindent\textbf{Effect of $\beta$ (geometric alignment).}
Figures~\ref{fig:beta_tail} and~\ref{fig:beta_angle} show that increasing $\beta$ enlarges the minimum inter-class angle $\theta_{\min}$ and improves Tail-AA robustness, with smooth and saturating trends consistent with the margin--robustness relationship in Theorem~\ref{thm:geom_margin}.
This indicates that geometric alignment acts as a stable inductive bias rather than brittle tuning.
As shown in Appendix Fig.~\ref{fig:beta_sweep_panel}, moderate $\beta$ improves tail robustness without degrading head-class performance, while overly large $\beta$ leads to over-regularization.



\subsection{RQ4: Mechanism Verification and Theory Consistency}
\label{sec:exp_rq4}


\noindent\textbf{Imbalance-induced geometric bias and Off-manifold adversarial drift.}
Figure~\ref{fig:exp_geom_alignment} reports geometry diagnostics.
Baseline methods exhibit reduced inter-class angular separation
and increased deviation from a margin-balanced ETF geometry.
In contrast, MCAT preserves larger angular margins
and maintains more balanced decision geometry.
Figure~\ref{fig:exp_manifold_drift} shows distributions of off-manifold drift.
Standard adversarial training induces pronounced drift,
especially for tail classes,
whereas MCAT substantially suppresses drift across all class groups.
In addition to quantitative diagnostics, we provide a qualitative case study by Figure~\ref{fig:appendix_case_study} in Appendix~\ref{sec:appendix_more_results}.
Compared to Base AT, MCAT yields noticeably tighter and better-separated tail-class embeddings while preserving compact head-class structure, offering intuitive evidence of improved geometric balance.

\noindent\textbf{Theory-aligned empirical evidence.}
Figure~\ref{fig:theory_aligned} provides theory-consistent observations:
(i) larger minimum inter-class angles correlate with stronger tail robustness,
(ii) MCAT shifts tail-class distributions of the sample-wise robustness proxy
toward larger values, and
(iii) increasing $\lambda$ jointly suppresses drift and improves robustness.
These results align with Theorems~\ref{thm:geom_margin},
\ref{thm:manifold_bound}, and Corollary~\ref{cor:sample_radius}.
Additional per-class results are deferred to Figure~\ref{fig:appendix_per_class_aa} in Appendix~\ref{sec:appendix_more_results}, where MCAT again shows the best results.

\section{Related Work}

\noindent\textbf{General Long-Tailed Learning.}
Long-tailed learning in the standard setting has been widely studied through data augmentation, training paradigms, and representation rebalancing.
Recent work leverages generative models for tail data synthesis~\cite{zhao2024ltgc,shao2024diffult}, controllable expert-based training~\cite{zhao2024prl}, and feature-space analyses that attribute tail failures to geometric distortion and representation collapse~\cite{yi2025geometrylt,sun2025rethinkretrain,zhou2024ccl}.

\noindent\textbf{Long-Tailed Adversarial Robustness.}
Prior studies show that adversarial training disproportionately harms tail classes under imbalance.
Existing solutions rely on margin or sampling rebalancing~\cite{wu2021adversarial,liu2022breadcrumbs}, staged or reweighted optimization~\cite{li2023reat,yu2025taet}, and robustness distillation~\cite{cho2025longtail}, but do not explicitly regulate feature geometry or semantic validity of adversarial examples.

\noindent\textbf{Geometry and Manifold Structure.}
Neural collapse and simplex ETF analyses highlight the role of balanced geometry in robustness~\cite{papyan2020prevalence,cao2025simplex,kothapalli2022neural,zhu2022balanced}, while manifold-based studies link adversarial vulnerability to off-manifold perturbations~\cite{li2025enhancing,satou2025geometrically,zhang2024manifold}.
Our work integrates these perspectives by jointly enforcing geometric balance and class-conditional manifold constraints for long-tailed adversarial training.

\section{Conclusion}

We proposed MCAT, a unified framework for long-tailed adversarial robustness that combines manifold-constrained adversarial training with ETF-inspired geometry regularization.
We provided theoretical results connecting balanced geometry to robust margins and showing the benefit of constraining adversarial drift away from semantic low-density regions.
Experiments on standard long-tailed benchmarks validate improved balanced robustness and tail performance under standard adversarial attacks.


\clearpage

\bibliographystyle{named}
\bibliography{MCAT}

\clearpage

\appendix

\section{Appendix: Proofs}
\label{app:proof}

\subsection{Notation}

We consider a classifier of the form
\[
f_\Theta(x)=W\phi_\Theta(x),
\]
where $W\in\R^{C\times m}$ is the linear classifier and $\phi_\Theta(x)\in\R^m$ is the feature representation.
Let $w_k$ denote the $k$-th row of $W$, and define the logit score for class $k$ as
\[
s_k(x)=w_k^\top \phi_\Theta(x).
\]

The $\ell_\infty$ adversarial ball is denoted by
\[
\mathcal{B}_\epsilon(x)=\{x'\mid \|x'-x\|_\infty\le\epsilon\}.
\]
The robust risk is
\[
R_{robust}(\theta)
=
\E_{(x,y)\sim\D}\Big[\max_{x'\in\mathcal{B}_\epsilon(x)}\ell(f_\Theta(x'),y)\Big],
\]
where $\ell(\cdot,\cdot)$ denotes the classification loss.


\subsection{Proof of Theorem~\ref{thm:geom_margin}}

\noindent\textbf{Assumptions.}
We assume normalized features and classifier weights, i.e.,
\[
\|\phi_\Theta(x)\|_2=1,
\qquad
\|w_k\|_2=1 \quad \forall k,
\]
which can be enforced without loss of generality.
We further assume that $\phi_\Theta$ is $L$-Lipschitz under $\ell_\infty$ perturbations, as stated in the theorem.

\begin{proof}
Fix a sample $(x,y)$ such that
\[
y=\arg\max_k s_k(x).
\]
Define the logit margin
\[
\gamma(x)=s_y(x)-\max_{k\ne y}s_k(x).
\]

\noindent\textbf{Step 1: Margin lower bound from ETF geometry.}
Since both $w_k$ and $\phi_\Theta(x)$ are unit-norm, we have
\[
s_k(x)=\cos\big(\angle(w_k,\phi_\Theta(x))\big).
\]
Let
\[
k^\star=\arg\max_{k\ne y}s_k(x).
\]
Under the approximate ETF assumption, the angle between $w_y$ and $w_{k^\star}$ is at least $\theta_{min}$.

The configuration minimizing the margin occurs when $\phi_\Theta(x)$ lies in the two-dimensional subspace spanned by $w_y$ and $w_{k^\star}$.
Elementary geometric arguments then yield
\[
\gamma(x)\ge \sin(\theta_{min}/2).
\]

\noindent\textbf{Step 2: Stability under adversarial perturbations.}
Let $x'=x+\delta$ with $\|\delta\|_\infty\le\epsilon$.
By the $L$-Lipschitz assumption,
\[
\|\phi_\Theta(x')-\phi_\Theta(x)\|_2\le L\epsilon.
\]
For any class $k$, we have
\[
|s_k(x')-s_k(x)|
=
|w_k^\top(\phi_\Theta(x')-\phi_\Theta(x))|
\le L\epsilon.
\]
Therefore, the margin at $x'$ satisfies
\[
\gamma(x')
\ge
\gamma(x)-2L\epsilon.
\]

\noindent\textbf{Step 3: Robustness condition.}
Combining the two bounds,
\[
\gamma(x')
\ge
\sin(\theta_{min}/2)-2L\epsilon.
\]
Thus, if
\[
\epsilon<\frac{\sin(\theta_{min}/2)}{L},
\]
the margin remains strictly positive and the predicted label is invariant to all perturbations in $\mathcal{B}_\epsilon(x)$.
This proves the theorem.
\end{proof}

\subsection{Proof of Corollary~\ref{cor:sample_radius}}

\begin{proof}
Fix a sample $(x,y)$ and define the logit score $s_k(x)=w_k^\top\phi_\Theta(x)$.
Let
\[
\gamma(x)=s_y(x)-\max_{k\ne y}s_k(x)
\]
denote the logit margin at $x$.

Consider any perturbation $x'=x+\delta$ with $\|\delta\|_\infty\le r$.
By the $L$-Lipschitz assumption on $\phi_\Theta$,
\[
\|\phi_\Theta(x')-\phi_\Theta(x)\|_2\le Lr.
\]
Assuming $\|w_k\|_2=1$ for all $k$, we have for each class $k$,
\[
|s_k(x')-s_k(x)|
=
|w_k^\top(\phi_\Theta(x')-\phi_\Theta(x))|
\le Lr.
\]
Let $k^\star=\arg\max_{k\ne y}s_k(x)$.
Then
\[
s_y(x')-s_{k^\star}(x')
\ge
\big(s_y(x)-Lr\big)-\big(s_{k^\star}(x)+Lr\big)
=
\gamma(x)-2Lr.
\]
Therefore, if $r\le \gamma(x)/(2L)$, the right-hand side remains non-negative, implying
\[
s_y(x')\ge s_{k^\star}(x')\ge s_k(x')\quad \forall k\ne y,
\]
and the predicted label is unchanged within the $\ell_\infty$ ball of radius $r$.
Thus the sample-wise robust radius satisfies
\[
r(x)\ge \frac{\gamma(x)}{2L}.
\]
\end{proof}

\subsection{Proof of Theorem~\ref{thm:manifold_bound}}

\noindent\textbf{Assumptions.}
We assume that the per-sample loss is bounded,
\[
0\le \ell(f_\Theta(x),y)\le \ell_{max},
\]
and that for each class $y$, the data distribution is supported on a semantic manifold $\mathcal{M}_y$ in feature space, while regions far from $\mathcal{M}_y$ carry negligible probability mass.

\begin{proof}
We formalize the intuition described in the main text.
Under long-tailed distributions, adversarial optimization may place excessive emphasis on perturbations whose features drift far away from the semantic manifold $\mathcal{M}_y$ of a tail class.
Although such off-manifold perturbations can induce high classification loss, they lie in low-density regions that are weakly supported by the data distribution and therefore do not meaningfully contribute to robustness on the semantic support.

Fix a sample $(x,y)$ and consider the inner maximization in the robust risk.
Let
\[
x^\star
=
\arg\max_{x'\in\mathcal{B}_\epsilon(x)}\ell(f_\Theta(x'),y).
\]

\noindent\textbf{Step 1: Decomposition of robust risk.}
To make the above distinction precise, we decompose the robust risk into on-manifold and off-manifold contributions:
\[
R_{robust}(\Theta)
=
R_{on}(\Theta)
+
R_{off}(\Theta),
\]
where $R_{on}$ corresponds to adversarial examples whose features remain within a neighborhood of the class manifold $\mathcal{M}_y$, and $R_{off}$ corresponds to adversarial examples whose features drift far away from $\mathcal{M}_y$.

By the manifold support assumption, off-manifold regions contribute negligible probability mass.
Therefore, there exists a constant $\rho\ll1$ such that
\[
R_{off}(\Theta)\le \rho\,\ell_{max}.
\]

\noindent\textbf{Step 2: Control of on-manifold risk via manifold-constrained objective.}
Recall the MCAT objective without the geometric regularizer:
\[
\begin{aligned}
R_{MCAT}(\Theta)
=
\E_{(x,y)\sim\D}\Big[
\max_{\|\delta\|_\infty\le\epsilon}\big(
\ell(f_\Theta(x+\delta),y) \\
\qquad\qquad\qquad\qquad
+ \lambda d_{\mathcal{M}_y}(\phi_\Theta(x+\delta))
\big)\Big].
\end{aligned}
\]
For adversarial examples whose features lie within a bounded neighborhood of $\mathcal{M}_y$, the manifold deviation term is uniformly bounded:
\[
d_{\mathcal{M}_y}(\phi_\Theta(x'))\le C.
\]
As a result, for such on-manifold adversarial points,
\[
\begin{aligned}
\ell(f_\Theta(x'),y)
&\le
\ell(f_\Theta(x'),y)+\lambda d_{\mathcal{M}_y}(\phi_\Theta(x')) \\
&\le
R_{MCAT}(\Theta)+\frac{C}{\lambda}.
\end{aligned}
\]

\noindent\textbf{Step 3: Combining the bounds.}
Taking expectation over $(x,y)$ and combining the on-manifold and off-manifold contributions yields
\[
R_{robust}(\Theta)
\le
R_{MCAT}(\Theta)
+
\frac{C}{\lambda}
+
\rho\,\ell_{max}.
\]
Since $\rho$ is negligible under the manifold support assumption, we conclude that
\[
R_{robust}(\Theta)\le R_{MCAT}(\Theta)+O(\lambda^{-1}),
\]
which completes the proof.
\end{proof}

\section{Generator Architecture}
\label{sec:appendix_generator}
Each class-conditional generator $G_y$ is implemented as a lightweight multilayer perceptron operating in the classifier feature space.
Given a latent code $z \in \mathbb{R}^{d_z}$ sampled from a standard Gaussian, $G_y$ outputs a feature vector in $\mathbb{R}^{d_f}$, where $d_f$ is the dimension of the penultimate-layer features of the backbone.

Concretely, $G_y$ consists of three fully connected layers with widths
$d_z \rightarrow 1024 \rightarrow 1024 \rightarrow d_f$.
ReLU activations are applied after the first two layers, and the output layer is linear.
No batch normalization or dropout is used.
All generators share the same architecture but are trained independently for each class.
Unless otherwise specified, we set $d_z=128$ and $d_f=512$ for CIFAR-based experiments.

\section{More Experimental Results}
\label{sec:appendix_more_results}
\begin{figure}[H]
  \centering
  \includegraphics[width=0.9\columnwidth]{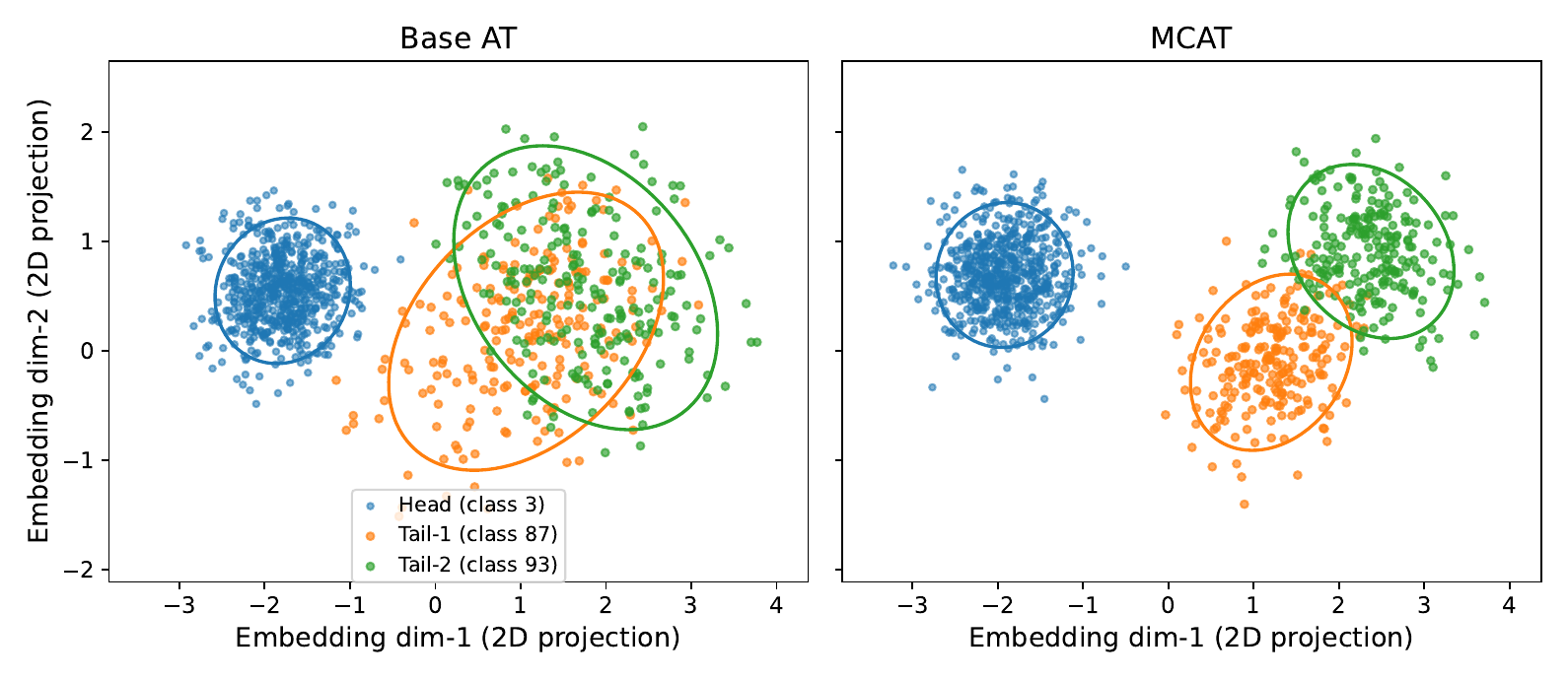}
\caption{
Case study on CIFAR-100-LT (IR=100) showing 2D embedding projections of one head and two tail classes.
Compared to Base AT, MCAT yields tighter and better-separated tail clusters while maintaining compact head representations.
}
  \label{fig:appendix_case_study}
\end{figure}

\begin{figure}[H]
  \centering
  \includegraphics[width=0.75\columnwidth]{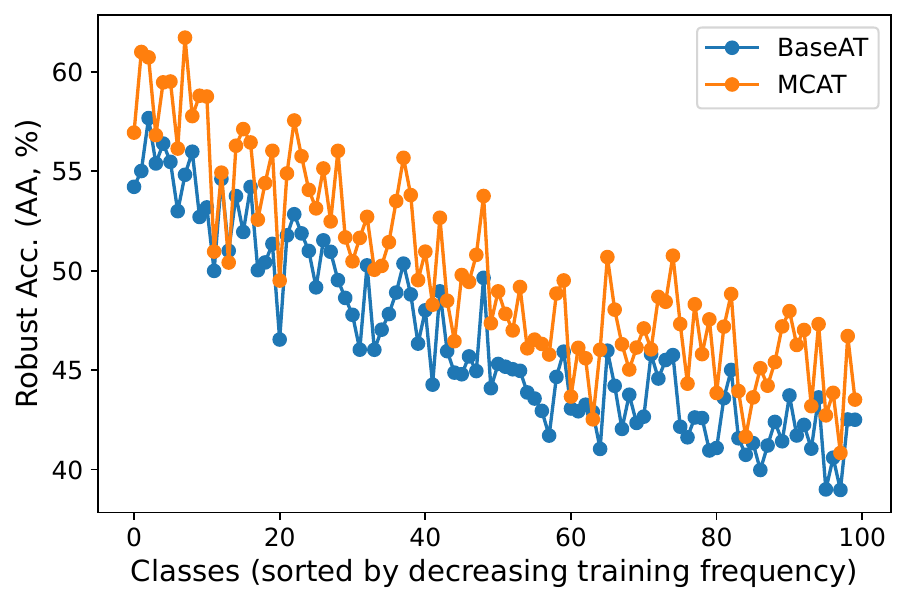}
  \caption{
Robust accuracy over all classes under AutoAttack (AA) on CIFAR-100-LT (IR=100), plotted against the sorted class frequency rank.
}
  \label{fig:appendix_per_class_aa}
\end{figure}

\begin{figure}[H]
\centering
\begin{subfigure}[t]{0.48\linewidth}
  \centering
  \includegraphics[width=\linewidth]{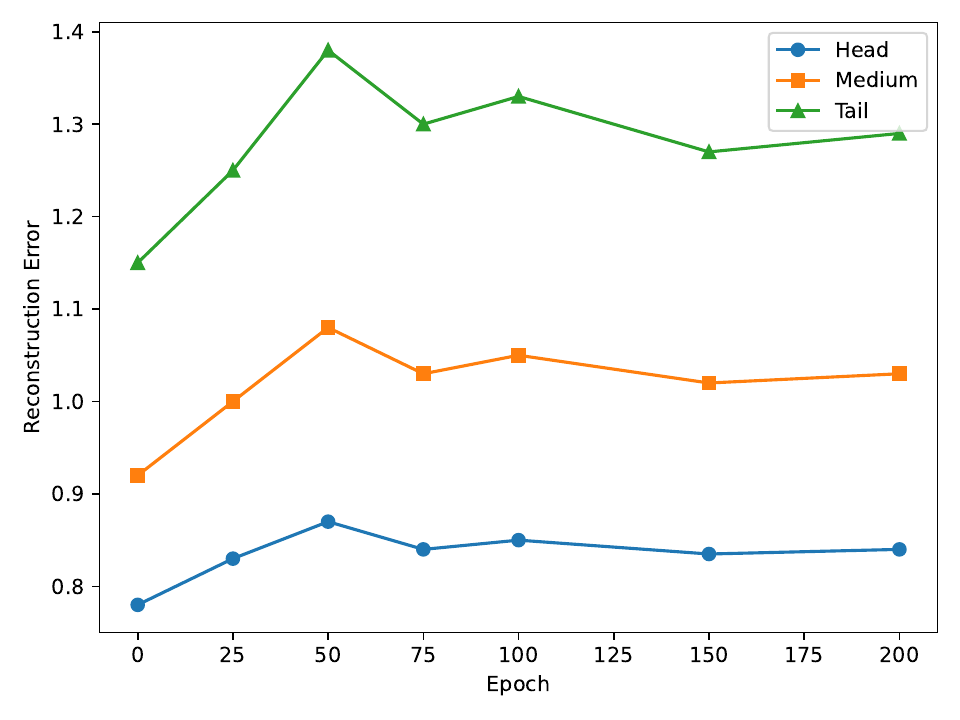}
  \caption{Reconstruction error $\|\phi_\Theta(x)-G_y(z^\star)\|_2$ over training epochs on CIFAR-100-LT (IR=100).
  The error exhibits mild non-monotonic fluctuations due to feature adaptation,
  but remains stable overall and does not diverge for tail classes.}
  \label{fig:recon_error_panel}
\end{subfigure}
\hfill
\begin{subfigure}[t]{0.48\linewidth}
  \centering
  \includegraphics[width=\linewidth]{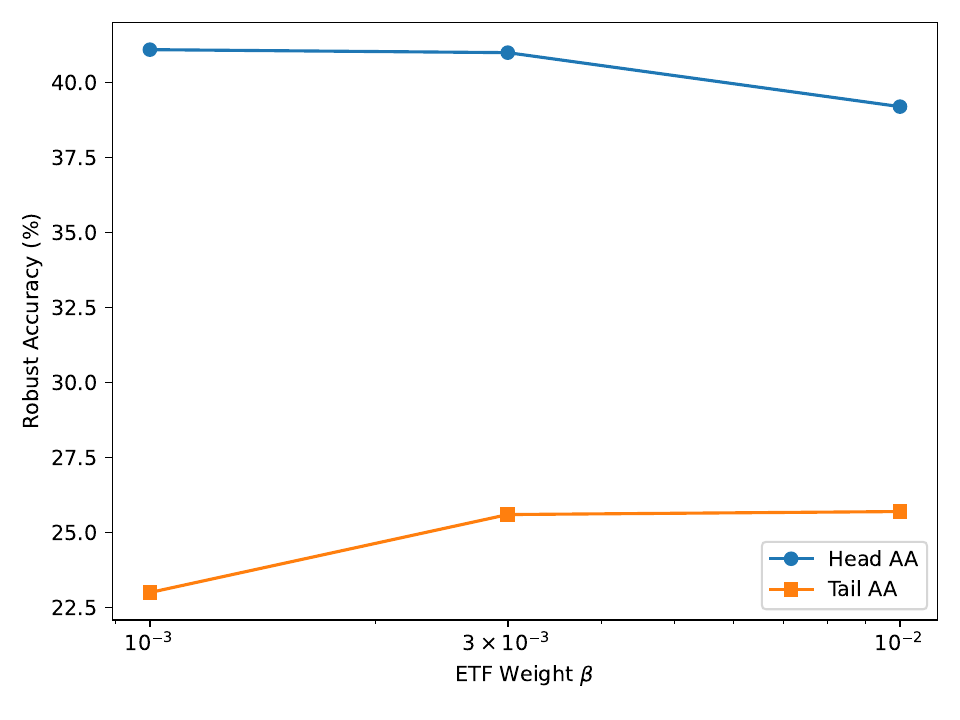}
  \caption{Effect of ETF regularization weight $\beta$ on robust accuracy for head and tail classes.
  Moderate $\beta$ improves tail robustness without degrading head performance,
  while excessively large $\beta$ leads to over-regularization.}
  \label{fig:beta_sweep_panel}
\end{subfigure}
\caption{Manifold stability and geometric alignment trade-off on CIFAR-100-LT (IR=100).
Left: reconstruction error of frozen class-conditional generators remains stable throughout adversarial training.
Right: ETF-inspired geometric alignment improves tail robustness without sacrificing head-class performance under moderate regularization.}
\label{fig:appendix_panel}
\end{figure}

\begin{table}[H]
\centering
\small
\setlength{\tabcolsep}{6pt}
\begin{tabular}{c|cccc}
\hline
Method & IR=10 & IR=20 & IR=50 & IR=100 \\
\hline
RoBal
& 31.80 & 29.10 & 26.40 & 25.90 \\

Self-Distill
& 33.40 & 31.20 & 29.10 & 28.10 \\

AT-BSL
& 32.90 & 30.80 & 28.70 & 27.30 \\
\hline
MCAT (ours)
& \textbf{34.60} & \textbf{33.10} & \textbf{32.00} & \textbf{31.50} \\
\hline
\end{tabular}
\caption{
AutoAttack robustness (\%) on CIFAR-100-LT under varying imbalance ratios.
MCAT exhibits consistently higher robustness and degrades more gracefully
as imbalance severity increases.
}
\label{tab:ir_sweep}
\end{table}

\begin{table}[H]
\centering
\small
\setlength{\tabcolsep}{6pt}
\begin{tabular}{c|cccc}
\hline
Method & IR=10 & IR=20 & IR=50 & IR=100 \\
\hline
RoBal
& 19.60 & 17.40 & 14.80 & 13.20 \\

Self-Distill
& 21.30 & 19.50 & 17.40 & 16.10 \\

AT-BSL
& 20.80 & 18.90 & 16.90 & 15.50 \\
\hline
MCAT (ours)
& \textbf{26.20} & \textbf{23.60} & \textbf{22.10} & \textbf{21.05} \\
\hline
\end{tabular}
\caption{
Tail-class AutoAttack robustness (\%) on CIFAR-100-LT under varying imbalance ratios.
MCAT consistently improves tail robustness and degrades more gracefully
as imbalance severity increases.
}
\label{tab:ir_sweep_tail_aa}
\end{table}

\begin{table}[H]
\centering
\small
\begin{tabular}{c|ccc}
\hline
$T_z$ & Overall AA (\%) & Tail AA (\%) & Relative Train Time \\
\hline
0 & 32.4 & 18.7 & 1.00 \\
1 & 34.9 & 21.5 & 1.07 \\
3 & 36.8 & 24.3 & 1.14 \\
5 & \textbf{37.5} & \textbf{25.6} & 1.18 \\
8 & 37.6 & 25.7 & 1.26 \\
\hline
\end{tabular}
\caption{Sensitivity to the number of latent optimization steps $T_z$ on CIFAR-100-LT (IR=100).}
\label{tab:tz_sensitivity}
\end{table}

\begin{table}[H]
\centering
\fontsize{7.3}{9}\selectfont
\setlength{\tabcolsep}{6pt}
\begin{tabular}{l|cc|cc}
\hline
Method
& BA $\uparrow$
& BR (AA) $\uparrow$
& Tail-PGD $\uparrow$
& Tail-AA $\uparrow$ \\
\hline
PGD-AT   & 31.20$\pm$0.60 & 11.40$\pm$0.65 &  9.30$\pm$0.70 &  8.10$\pm$0.75 \\
TRADES   & 32.10$\pm$0.58 & 12.80$\pm$0.62 & 10.20$\pm$0.68 &  8.90$\pm$0.73 \\
MART     & 31.80$\pm$0.59 & 13.20$\pm$0.63 & 10.60$\pm$0.69 &  9.20$\pm$0.74 \\
AWP      & 32.90$\pm$0.55 & 14.10$\pm$0.60 & 11.30$\pm$0.65 & 10.10$\pm$0.70 \\
\hline
RoBal        & 35.60$\pm$0.52 & 15.90$\pm$0.58 & 12.80$\pm$0.62 & 11.20$\pm$0.67 \\
REAT         & 36.80$\pm$0.50 & 17.10$\pm$0.56 & 13.90$\pm$0.60 & 12.30$\pm$0.65 \\
TAET         & 37.40$\pm$0.48 & 17.80$\pm$0.55 & 14.60$\pm$0.59 & 12.90$\pm$0.64 \\
Self-Distill & 37.90$\pm$0.49 & 18.20$\pm$0.54 & 15.00$\pm$0.58 & 13.40$\pm$0.63 \\
AT-BSL       & 37.60$\pm$0.50 & 17.90$\pm$0.55 & 14.70$\pm$0.59 & 13.10$\pm$0.64 \\
\hline
MCAT (ours)  & \textbf{42.30$\pm$0.45} & \textbf{22.60$\pm$0.50} & \textbf{18.90$\pm$0.54} & \textbf{16.80$\pm$0.58} \\
\hline
\end{tabular}
\caption{
Balanced and tail robustness on \textbf{Tiny-ImageNet-LT (IR=100)}.
We report balanced accuracy (BA) and balanced robustness (BR),
defined as average per-class accuracy under clean evaluation and AutoAttack (AA), respectively,
together with tail-class robust accuracy under PGD-20 (Tail-PGD) and AutoAttack (Tail-AA).
Results are reported as mean$\pm$std over three random seeds.
}
\label{tab:tiny_balanced_tail}
\end{table}


\section{Hyperparameter Settings}
\label{app:hyper}

\begin{table}[t]
\centering
\fontsize{6.8}{6.5}\selectfont
\setlength{\tabcolsep}{2pt}
\begin{tabularx}{\linewidth}{l X X}
\toprule
\textbf{Category} & \textbf{Hyperparameter} & \textbf{Value} \\
\midrule
\multirow{6}{*}{Training}
 & Optimizer & SGD with momentum \\
 & Momentum & 0.9 \\
 & Weight decay & $5\times10^{-4}$ \\
 & Batch size & 128 \\
 & Training epochs & 200 \\
 & Learning rate schedule & Cosine decay \\
\midrule
\multirow{4}{*}{Adversarial Setup}
 & Threat model & $\ell_\infty$ \\
 & Perturbation budget $\epsilon$ & $8/255$ \\
 & Step size $\alpha$ & $2/255$ \\
 & PGD steps & 10 (train), 20 (eval) \\
\midrule
\multirow{3}{*}{Backbone}
 & Architecture & ResNet-18 \\
 & Initialization & He initialization \\
 & Normalization & BatchNorm \\
\midrule
\multirow{4}{*}{Long-tailed Setting}
 & Imbalance type & Exponential \\
 & Imbalance ratio (IR) & $\{10, 50, 100\}$ \\
 & Sampling strategy & Class-uniform \\
 & Evaluation metric & Overall / Many / Medium / Few \\
\midrule
\multirow{6}{*}{MCAT (Ours)}
 & Manifold penalty weight $\lambda_{\mathrm{man}}$ & 0.1 \\
 & Geometric alignment weight $\beta$ & $3\times10^{-3}$ \\
 & Equivalent $\lambda_{\mathrm{geom}}$ in implementation & 0.01 \\
 & Manifold distance metric & $\ell_2$ in feature space \\
 & Manifold update frequency & Every iteration \\
 & ETF target dimension & Equal to number of classes \\
\midrule
\multirow{5}{*}{Baselines}
 & AT / TRADES & Official recommended settings \\
 & RoBal & Margin reweighting as in~\cite{wu2021adversarial} \\
 & REAT & Loss reweighting as in~\cite{li2023reat} \\
 & TAET & Two-stage schedule as in~\cite{yu2025taet} \\
 & Distillation-based & Teacher trained on balanced AT \\
\bottomrule
\end{tabularx}
\caption{Hyperparameter settings for MCAT and baseline adversarial training methods.
$\beta$ denotes the ETF-inspired geometric alignment weight used in Eq.~(3).
Unless otherwise specified, all methods share the same backbone architecture and training protocol.}
\label{tab:hyperparams}
\end{table}

\end{document}